%% file: main.tex
\documentclass[letterpaper]{article} 
\usepackage{aaai24}  
\usepackage{times}  
\usepackage{helvet}  
\usepackage{courier}  
\usepackage[hyphens]{url}  
\usepackage{graphicx} 
\usepackage{natbib}  
\usepackage{caption} 
\usepackage{algorithm}
\usepackage{algorithmic}
\usepackage{booktabs}
\usepackage{amsthm}
\usepackage{amsmath}

\newtheorem{theorem}{Theorem}
\newtheorem{proposition}{Proposition}
\newtheorem{remark}{Remark}
\newtheorem{lemma}{Lemma}

\newtheorem{definition}{Definition}
\newtheorem{corollary}{Corollary}

\input{MathMacros.tex}

\usepackage{newfloat}
\usepackage{listings}
\DeclareCaptionStyle{ruled}{labelfont=normalfont,labelsep=colon,strut=off} 
\floatstyle{ruled}
\newfloat{listing}{tb}{lst}{}
\floatname{listing}{Listing}
\title{The Expected Loss of Preconditioned Langevin Dynamics\\Reveals the Hessian Rank}
\author{
    Amitay Bar\equalcontrib,
    Rotem Mulayoff\equalcontrib,
    Tomer Michaeli,
    Ronen Talmon
}
\affiliations{
    Technion -- Israel Institute of Technology, Haifa, Israel\\
    amitayb@campus.technion.ac.il
}

\begin{document}

\maketitle

\input{Abstract.tex}

\input{Introduction.tex}

\input{ProblemSetting.tex}

\input{TheoreticalAnalysis.tex}

\input{ExperimenralResults.tex}

\input{RelatedWork.tex}

\input{Conclusions.tex}

\input{Acknowledgment.tex}

\bibliography{Refs}

\onecolumn
\appendix

\input{Proofs.tex}

\input{AuxiliaryLemmas.tex}

\input{HessianTraceEstimation.tex}

\input{AdditionalExperimentalResults.tex}

\end{document}

%% file: MathMacros.tex
\usepackage{amsmath,amsfonts,bm}

\def\eqref#1{equation~\ref{#1}}

\def\1{\bm{1}}

\def\vtheta{{\bm{\theta}}}
\def\vpsi{{\bm{\psi}}}

\def\vn{{\bm{n}}}

\def\vu{{\bm{u}}}
\def\vv{{\bm{v}}}

\def\vx{{\bm{x}}}
\def\vy{{\bm{y}}}

\def\mG{{\bm{G}}}
\def\mH{{\bm{H}}}
\def\mI{{\bm{I}}}
\def\mJ{{\bm{J}}}

\def\mM{{\bm{M}}}

\def\mP{{\bm{P}}}

\def\mS{{\bm{S}}}

\def\mW{{\bm{W}}}

\def\mPhi{{\bm{\Phi}}}
\def\mLambda{{\bm{\Lambda}}}
\def\mSigma{{\bm{\Sigma}}}

\DeclareMathAlphabet{\mathsfit}{\encodingdefault}{\sfdefault}{m}{sl}
\SetMathAlphabet{\mathsfit}{bold}{\encodingdefault}{\sfdefault}{bx}{n}

\def\sR{{\mathbb{R}}}

\DeclareMathOperator*{\argmax}{arg\,max}

\newcommand{\Tr}{\mathrm{Tr}}

\newcommand{\rank}{\mathrm{rank}}

%% file: Abstract.tex
\begin{abstract}
    Langevin dynamics (LD) is widely used for sampling from distributions and for optimization. In this work, we derive a closed-form expression for the expected loss of preconditioned LD near stationary points of the objective function. We use the fact that at the vicinity of such points, LD reduces to an Ornstein–Uhlenbeck process, which is amenable to convenient mathematical treatment. Our analysis reveals that when the preconditioning matrix satisfies a particular relation with respect to the noise covariance, LD's expected loss becomes proportional to the rank of the objective's Hessian. We illustrate the applicability of this result in the context of neural networks, where the Hessian rank has been shown to capture the complexity of the predictor function but is usually computationally hard to probe. Finally, we use our analysis to compare SGD-like and Adam-like preconditioners and identify the regimes under which each of them leads to a lower expected loss.
\end{abstract}

%% file: Introduction.tex
\section*{Introduction}
Langevin dynamics (LD) has proven to be a powerful tool across many domains. Its basic discretization, the unadjusted Langevin algorithm (ULA), along with other variants, such as the stochastic gradient Langevin dynamics (SGLD) method \citep{welling2011bayesian} and its extensions, are commonly used for sampling from distributions \cite{ding2014bayesian,wang2015privacy} and for nonconvex optimization \citep{gelfand1991recursive, raginsky2017non,xu2018global, chen2020stationary, borysenko2021coolmomentum}. LD and similar stochastic differential equations (SDEs) are also used for analyzing the optimization process of neural networks (NNs), as they serve as continuous-time analogues to popular optimizers, like SGD \cite{arora2018optimization, elkabetz2021continuous, latz2021analysis, zhu2019anisotropic}. 

An important variant of SGLD is the stochastic gradient Riemannian Langevin dynamics (SGRLD)  method \citep{patterson2013stochastic}, which is an LD-type random process on a Riemannian manifold that respects the Riemannian metric. \citet{girolami2011riemann} used Riemannian LD as an improved Markov chain Monte Carlo (MCMC) method for sampling from distributions. Here we view SGRLD as a preconditioned version of LD, where the Riemannian metric tensor plays the role of the preconditioner. This more general viewpoint is of practical value, as  preconditioning is commonly used for circumventing instabilities that stem from an ill-conditioned loss landscape, in optimization problems in general \cite{pock2011diagonal,dauphin2015equilibrated}, and in NN training in particular  \cite{kingma2014adam, Tieleman2012RMSProp}. Combining preconditioning and SGLD has been previously studied in the past \cite{li2016preconditioned,marceau2017natural}. For example, \citet{marceau2017natural} used a Fisher matrix approximation for choosing the variance of the noise in preconditioned SGLD.

In this paper, we consider a preconditioned LD and study its expected loss near stationary points on a large time scale. Specifically, we use a quadratic approximation of the loss about the stationary point, leading to an Ornstein–Uhlenbeck (OU) process, which is mathematically tractable. We show that when the preconditioner and the noise covariance satisfy a particular relation, the expected loss is linear in the Hessian rank. The Hessian of the loss of NNs has attracted a lot of interest in recent years, as it has been linked to model ``complexity'' \cite{arora2019implicit,li2020towards} and generalization \cite{huh2022low} and has been shown to exhibit interesting spectral properties \cite{papyan2020traces}. Yet, it is typically infeasible to compute the rank of the Hessian of a NN (or even just store the Hessian) due to the large number of parameters. Our theoretical result suggests that we can estimate the Hessian rank at minima by simply applying LD in its vicinity. Leveraging our result, we devise an iterative Hessian rank estimation algorithm, which does not require spectral decomposition of the Hessian. Remarkably, the preconditioning matrix and the covariance matrix revealing the Hessian rank lead to a \emph{Riemannian} LD with a Riemannian metric that is equal to the inverse of the preconditioner.

Additionally, we show that under certain conditions, the expected loss at a large time scale depends only on the interplay between the preconditioner and the noise covariance. This simple result enables us to analyze the impact of different preconditioners on the expected loss. Specifically, we examine two preconditioning matrices which correspond to stochastic gradient descent (SGD) \cite{robbins1951stochastic} and Adam \cite{kingma2014adam}. Next, we derive conditions on the preconditioning that lead to the maximal expected loss. From the standpoint of NNs with a non-convex objective function, higher loss indicates better escaping efficiency from local minima \cite{zhu2019anisotropic}. For completeness, we also consider initialization at a saddle point and show that another derivative of our analysis is the ability of the preconditioned LD to escape saddle points.

We empirically demonstrate our theory on linear and nonlinear NNs. We show that the expected loss of the networks incorporated into preconditioned LD with specific preconditioning results in an accurate estimation of the Hessian rank.

%% file: ProblemSetting.tex
\section*{Preconditioned Langevin Dynamics}
We consider a preconditioned LD given by the SDE 
\begin{equation}\label{eq: SDE G gradient G noise}
    \smash{d\vtheta_t = -\mG\nabla f(\vtheta_t)dt +  \mG\mSigma^{\frac{1}{2}} d\vn_t,}
\end{equation}
where $\vtheta_t\in\sR^n$, $f(\cdot)$ is a scalar objective function and $\nabla f(\vtheta_t)$ is its gradient, $\mG\in\mS_{++}^n$ is a positive definite (PD) preconditioning matrix, $\mSigma$ is a noise covariance matrix, and $d\vn_t$ is a standard Brownian motion. The SDE in~(\ref{eq: SDE G gradient G noise}) stems from the Riemannian LD presented in \cite{girolami2011riemann, patterson2013stochastic}. Specifically, when $\mSigma=\mG^{-1}$, the SDE in (\ref{eq: SDE G gradient G noise}) coincides with LD on a flat Riemannian manifold, where $\mG^{-1}$ is the metric tensor matrix. Note that in (\ref{eq: SDE G gradient G noise}), we decouple the preconditioning matrix $\mG$ and the noise covariance $\mSigma$. This allows us to model a user-chosen preconditioning matrix $\mG$ multiplying gradient estimation noise with covariance $\mSigma$.

The Riemannian LD has been considered in many studies. For example, \citet{girolami2011riemann} used it for improving sequential MCMC algorithms, and \citet{li2016preconditioned} used it for efficient sampling from the posterior. Broadly, SDEs such as (\ref{eq: SDE G gradient G noise}) naturally model gradient flow schemes, which are often viewed as the continuous counterparts of gradient descent schemes. Therefore, such SDEs were recently used for analyzing the optimization process of NNs \cite{arora2018optimization, elkabetz2021continuous, latz2021analysis}. In that context, $\vtheta$ is the vector of parameters of the network, $f(\vtheta) = \mathcal{L}(\vtheta)$ is the loss function, $\mG$ is the user-chosen preconditioner, and $d\vn$ accounts for the gradient noise that arises due to the use of small batches. Alternatively, $d\vn$ can be synthetic noise added to the gradient for better learning and generalization \cite{neelakantan2015adding,kaiser2015neural,zeyer2017comprehensive}. 

We explore the effect of preconditioning and the Hessian rank on the expected loss of the preconditioned LD. We remark that selecting the preconditioner $\mG$ embodies the flexibility to select its magnitude and direction. Since the magnitude can be seen as the continuous counterpart of the adjustment of the discrete stepsize (learning rate), we maintain the norm of the preconditioner fixed, thereby allowing for the examination of the impact of the direction. 

To make the analysis tractable, we consider the second-order approximation of $f(\vtheta)$ around a stationary point, $\vtheta^*$, for which $\nabla f(\vtheta^*)=0$, namely, 
\begin{equation}\label{eq:NQM}
    f(\vtheta) \approx f(\vtheta^*) + \frac{1}{2}(\vtheta-\vtheta^*)^T\mH(\vtheta-\vtheta^*),
\end{equation}
where $\mH=\nabla^2f(\vtheta^*)\in\mS^n(\sR)$ is the Hessian of $ f $ at~$ \vtheta^* $. Without loss of generality, we assume that $\vtheta^*=\mathbf{0}$ and $f(\vtheta^*)=0$ \citep{chen2020stationary,zhu2019anisotropic}. Such noisy quadratic models (\ref{eq:NQM}) were considered in the past to model NN optimization, and it was shown that despite their simplicity, they capture important features of non-trivial NNs that are used in practice \cite{zhang2019algorithmic}. We note that the second-order approximation is made only for analysis and not in our experiments.

Under the noisy quadratic model at the vicinity of $\vtheta^*$, the SDE in (\ref{eq: SDE G gradient G noise}) becomes the following multivariate OU process \cite{gardiner1985handbook}, given by
\begin{equation}\label{eq: OU process for dx_t}
    d\vtheta_t =
    -\mG\mH\vtheta_tdt +
    \mG\mathbf{\mSigma}^{\frac{1}{2}} d\vn_t.
\end{equation}
An analysis of (\ref{eq: OU process for dx_t}) without preconditioning, \emph{i.e.}, for $\mG=\mI$, appears in \cite{zhu2019anisotropic} for a short time scale. In contrast, here, we specifically analyze the effect of the preconditioning over long-time scales.

We conclude this section by noting that the preconditioner in (\ref{eq: SDE G gradient G noise}) multiplies both the gradient and the noise. In case we have access to the accurate gradient, it was proposed to add noise to escape local minima and saddle points \cite{chen2020stationary,choi2023appropriate}. This leads to the following SDE, where $\mG$ multiplies only the gradient
\begin{equation}\label{eq: SDE G multiplied only the gradient}
    d\vtheta_t = -\mG\nabla \mathcal{L}(\vtheta_t)dt +  \mSigma^{\frac{1}{2}} d\vn_t.
\end{equation}
With only slight changes to our analysis, similar results could be obtained for the SDE in (\ref{eq: SDE G multiplied only the gradient}) as well.

%% file: TheoreticalAnalysis.tex
\section*{Analyzing the Expected Loss}
In this section, we analyze the expected loss induced by the SDE in (\ref{eq: OU process for dx_t}) and present the effects of different preconditioning matrices. The proofs appear in the supplementary materials (SM) included in the arXiv version of the paper. 

We begin by examining the expected loss near a stationary point induced by the LD in (\ref{eq: OU process for dx_t}) for a loss function with an arbitrary Hessian matrix $\mH$, representing either a minimum point or a saddle point.
\begin{theorem}[\textbf{Expected loss over time}] \label{thm: E[f(x)] expression}
    The expected loss of a process governed by the SDE in (\ref{eq: OU process for dx_t}) is given by
    \begin{equation} \label{eq: E[f(x)] for minima}
        \mathbb{E}[f(\vtheta_t)] = 
        \frac{1}{4}\Tr\left(\mSigma \mG\left(\mI - e^{-2\mG\mH t}\right)\right).
    \end{equation}
\end{theorem}
The expected loss has been previously investigated considering $\mG=\mI$. Indeed, when setting $\mG=\mI$ in (\ref{eq: E[f(x)] for minima}), the result coincides with the result presented in \cite{zhu2019anisotropic,chen2020stationary}. Considering a general preconditioner $\mG$ requires a more involved derivation since the symmetry of the matrices breaks. Our key observation is that this can be circumvented by using matrix similarity. See more details in the SM.

Next, we consider a minimum point with a positive semi-definite (PSD) Hessian, namely, the Hessian eigenvalues are larger than or equal to zero. Taking the limit of $t\rightarrow \infty$ in~(\ref{eq: E[f(x)] for minima}) leads to our main result.
\pagebreak
\begin{proposition}[\textbf{Expected loss for large $t$}]\label{thm: E[f(x)] expression for large t}
    When the Hessian is PSD, %
    we have:
    \begin{equation}\label{eq: E[f(x)] for large t}
        \lim_{t\rightarrow\infty}\mathbb{E}[f(\vtheta_t)]
        = \frac{1}{4}\Tr\left(\mSigma \mG \mP\mJ\mP^{-1}\right),
    \end{equation}
    where $\mJ$ is the following diagonal matrix 
    \begin{equation}
        \mJ_{ii} = 
        \begin{cases}
          1 & \lambda_i\{\mG\mH\} > 0\\
          0 & \lambda_i\{\mG\mH\} = 0 
        \end{cases},\quad i=1,\ldots,n,
    \end{equation}
    and $\mP$ is a matrix whose columns are the eigenvectors of $\mG\mH$, ordered according to the eigenvalues. 
\end{proposition}
Note that $\mG\mH$ is similar to the symmetric matrix $\mG^{\frac{1}{2}}\mH\mG^{\frac{1}{2}}$, and therefore, it has a real spectrum.

Proposition \ref{thm: E[f(x)] expression for large t} provides means to prove the interplay between the expected loss, the Hessian matrix $\mH$, the covariance of the noise $\mSigma$, and the (user-chosen) preconditioner $\mG$. The limit of $t\rightarrow\infty$ in (\ref{eq: E[f(x)] for large t}) is required so that $e^{-2\lambda_{\min^+}\{\mG\mH\}t} \ll 1$, where $\lambda_{\min^+}\{\mG\mH\}$ is the smallest non-zero eigenvalue of $\mG\mH$. This holds for sufficiently large $t$ that satisfies $t\gg\ 1/\lambda_{\min^+}\{\mG\mH\}$.

We demonstrate the theoretical results in Figure \ref{fig: Loss vs iterations linear Nets}, which presents the loss of a linear NN. Since the Hessian of such a network has an explicit expression \cite{mulayoff2020unique}, we are able to present the theoretical expression of the expected loss over time according to Theorem \ref{thm: E[f(x)] expression}, and the theoretical steady-state value according to Proposition \ref{thm: E[f(x)] expression for large t}. See details in the Application to Hessian Rank Estimation section.

\begin{figure}
    \centering
    \includegraphics[width=0.8\columnwidth]{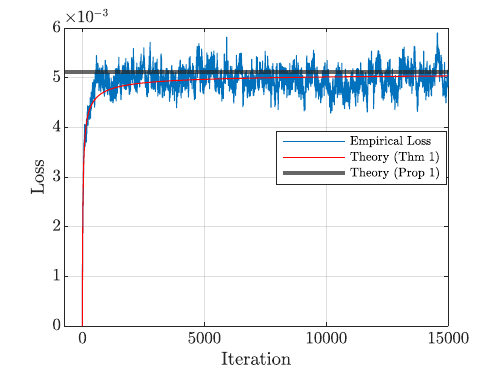} %
    \caption{The loss (in blue) for a linear NN of depth $5$ with input and output dimensions of $32$ (the number of parameters is $5120$). The theoretical expressions according to Theorem \ref{thm: E[f(x)] expression} and Proposition \ref{thm: E[f(x)] expression for large t} are in red and black, respectively.}
    \label{fig: Loss vs iterations linear Nets}
\end{figure}

When the Hessian is strictly positive (\emph{i.e.}, a PD matrix), the following result stems directly from Proposition~\ref{thm: E[f(x)] expression for large t}. 
\begin{corollary}[\textbf{Expected loss for PD Hessian}]   
\label{prop: E[f(X)] --> tr(Sigma G)}
If the Hessian is PD, then
\begin{equation} \label{eq: E[f(x)] for minima for large t}
    \lim_{t\rightarrow \infty}\mathbb{E}[f(\vtheta_t)]
    = \frac{1}{4}\Tr\left(\mSigma \mG\right).
\end{equation}
\end{corollary}
According to Corollary \ref{prop: E[f(X)] --> tr(Sigma G)}, when $t\rightarrow\infty$, the expected loss reaches a fixed value depending only on the interplay between the preconditioner $\mG$ and the noise covariance $\mSigma$.

From (\ref{eq: E[f(x)] for minima for large t}), we see that there exist infinitely many preconditioners that result in the same expected loss because the trace operation imposes only $n$ constraints on an $n\times n$ matrix. In particular, any expected loss can be achieved by a diagonal preconditioning matrix. This is important in the context of deep NN, when the number of parameters is typically very large, and adaptive methods resort to diagonal preconditioning matrices (for example, Adam \cite{kingma2014adam}).

\subsection*{The Preconditioner Revealing Hessian Rank}
An important consequence of Proposition \ref{thm: E[f(x)] expression for large t} is that the expected loss reveals the rank of the Hessian for a particular choice of the preconditioner. Specifically, by setting $\mG\mSigma = \sigma^2\mI$ in Proposition \ref{thm: E[f(x)] expression for large t}, we get that the expected loss tends to $ \sigma^2\Tr(\mJ) /4 $. Note that $ \Tr(\mJ) $ is the rank of $ \mG \mH $. Now, since $ \mG $ has full rank, we have that $ \rank(\mG \mH) = \rank(\mH) $, which leads us to the following result.
\begin{corollary}[\textbf{Expected loss and Hessian rank}]\label{thm: E[f(x)] --> rank(H)}
If $\mSigma\mG = \sigma^2\mI$, then 
\begin{equation}\label{eq: rank(H) = Ef(x)}
    \mathrm{rank}(\mH)
    = \lim_{t\rightarrow \infty}\frac{4}{\sigma^2}\mathbb{E}[f(\vtheta_t)].
\end{equation}
\end{corollary}
The importance of Corollary \ref{thm: E[f(x)] --> rank(H)} is that it links the Hessian rank, a meaningful quantity that is typically hard to estimate, with the expected loss, which can be computed. In other words, Corollary \ref{thm: E[f(x)] --> rank(H)} prescribes a way to estimate the Hessian rank by observing the expected loss.  
\begin{remark}
    Setting $\mSigma\mG = \sigma^2\mI$ in the SDE in (\ref{eq: SDE G gradient G noise}) leads to the following Riemannian LD on a flat manifold \cite{girolami2011riemann}
    \begin{equation}
        d\vtheta_t = -\mG\nabla f(\vtheta_t)dt +  \sigma\sqrt{\mG}d\vn_t,
    \end{equation}
    where $\mG^{-1}$ is the Riemannian metric.
\end{remark}
In other words, we see that the specific relationship between the gradient and the noise that describes a ``natural'' process on a manifold reveals the rank. Since the Hessian rank could potentially be large, in practice, we keep the quadratic approximation valid by setting $\sigma^2\sim\frac{1}{d}$. We note that the preconditioner $\mG$ is important in circumventing instabilities. In the Application to Hessian Rank Estimation section, we describe in detail the algorithm for Hessian rank estimation that is based on Corollary \ref{thm: E[f(x)] --> rank(H)}.

Next, we examine the effect of the Hessian rank on the expected loss, while considering arbitrary preconditioners. Here we fix $\mG$ along with $\mSigma$ and consider two different Hessian matrices with different ranks.
\begin{proposition}
\label{prop: rank H and loss}
    For the same preconditioner $\mG$ and noise covariance matrix $\mSigma$, if 
    \begin{equation}
        \rank(\mH_1) \le \rank(\mH_2),   
    \end{equation}    
    then    
    \begin{equation}
        \lim_{t\rightarrow \infty}\mathbb{E}[f_1(\vtheta_t)]
        \leq \lim_{t\rightarrow \infty}\mathbb{E}[f_2(\vtheta_t)]. 
    \end{equation}
\end{proposition}
According to this proposition, a larger Hessian rank indicates a higher asymptotic expected loss. In the context of NNs, Proposition \ref{prop: rank H and loss} could imply that, typically, \emph{when a NN reaches a low loss during training, the loss has a low-rank Hessian}. In turn, for a low-rank Hessian, the noise directed at the null space of the Hessian does not affect the loss.  

The expected loss is related not only to the Hessian rank but also to its trace. Following Theorem \ref{thm: E[f(x)] expression}, the derivative of the expected loss in (\ref{eq: E[f(x)] for minima}) with respect to the time $t$ is
\begin{equation}
    \frac{\partial}{\partial t} \mathbb{E}[f(\vtheta_t)] = 
    \frac{1}{2}\Tr\left( 
        \mSigma \mG\mG\mH e^{-2\mG\mH t}
    \right).
\end{equation}
By setting $\mG=\mSigma^{-\frac{1}{2}}$ and $t=0$, we obtain
\begin{equation}
\label{eq: d/dt of Ef(x) = trace H}
    \frac{\partial}{\partial t} \mathbb{E}[f(\vtheta_t)]\Big|_{t=0} = 
     \frac{1}{2}\Tr\left(         \mH 
    \right).
\end{equation}
Thus, for the particular preconditioning matrix $\mG=\mSigma^{-\frac{1}{2}}$, the derivative of the expected loss at time $t=0$ is proportional to the Hessian trace. See more details in the SM.

\subsection*{Specific Preconditioners and Their Effect on the Expected Loss}
The expected loss for $t\rightarrow\infty$ depends only on the interplay between the preconditioner and the noise covariance according to Corollary \ref{prop: E[f(X)] --> tr(Sigma G)}\footnote{The results in this subsection are for $t\rightarrow\infty$.}. We fix the magnitude of the preconditioner and examine the direction leading to the maximal expected loss. The expected loss is indicative of the ability to escape local minima. In \cite{zhu2019anisotropic}, it is used to define the escaping efficiency as $\mu_t\equiv \mathbb{E}_{\vtheta_t}[f(\vtheta_t)-f(\vtheta_0)]$, where $\vtheta_0$ is the value of $\vtheta_t$ for $t=0$. We note that higher expected loss means better escaping efficiency.
\begin{corollary}[\textbf{Maximal expected loss}]\label{prop: G for maximal expected loss }
    The preconditioner $\mG$ leading to the maximal expected loss is proportional to the noise covariance matrix. Formally,
    \begin{equation}
        \mG^*
        = \argmax_{\mG \; \text{s.t.} \; \|\mG\|_F=1} \mathbb{E}[f(\vtheta_t)]
        = \frac{\mSigma}{\|\mSigma\|_F}.
    \end{equation}
\end{corollary}
Corollary \ref{prop: G for maximal expected loss } holds since the trace is an inner product, and hence, the maximum of $\Tr\left(\mSigma \mG\right)$ is achieved for $\mG\propto \mSigma$. In other words, we see that a preconditioner that is aligned with the noise covariance results in the highest expected loss.

For PD matrices $\mH$ and $\mG$ we have $ \mathbb{E}[f(\vtheta_t)] = \Tr(\mSigma\mG) = \Tr(\mG^{\frac{1}{2}}\mSigma\mG^{\frac{1}{2}}) > 0$. This means that
\begin{equation}
        \mathbb{E}[f(\vtheta_t)] > f(\vtheta_0),
\end{equation}
and after initialization, the expected loss is greater than its initial value.

Next, we focus on two choices of preconditioning matrices: $\mG=\mI$ and $\mG=\mSigma^{-\frac{1}{2}}$ and examine the consequent expected loss. For this purpose, consider two stochastic processes $\vtheta_t$ and $\vpsi_t$ that follow the SDE in (\ref{eq: OU process for dx_t}) with preconditioning matrices $\mG=\mI$ and $\mG=\mSigma^{-\frac{1}{2}}$, respectively. We note that $\vtheta_t$ can be viewed as the continuous counterpart of the SGD algorithm \cite{robbins1951stochastic} since the gradient comprises two terms, the accurate gradient $\nabla f$, and a noise term. The process $\vpsi_t$ is similar in spirit to Adam \cite{kingma2014adam} since it also considers the square root of the second-order statistics, but instead of the correlation matrix, it makes use of the covariance of the noise. Setting $\mG=\mI$ and $\mG=\mSigma^{-\frac{1}{2}}$ in (\ref{eq: E[f(x)] for minima for large t}) results in $\mathbb{E}[f(\vtheta_t)] = \Tr(\mSigma)$ and $\mathbb{E}[f(\vpsi_t)] = \Tr(\mSigma^{\frac{1}{2}})$, respectively. To determine which preconditioner leads to a larger loss, we compare between $\Tr(\mSigma)$ and $\Tr(\mSigma^{\frac{1}{2}})$. Generally, when the noise is high,  $\Tr(\mSigma) > \Tr(\mSigma^{\frac{1}{2}})$ and $\mG=\mI$ leads to a higher loss.

We further investigate the two preconditioners by fixing their magnitude and examining the effect of their direction on the loss. 
\begin{proposition}\label{prop: Ef(x) > Ef(y) two preconditioning} 
    Suppose the preconditioners of $\vtheta_t$ and $\vpsi_t$ have the same Frobenius norm. If
    \begin{equation}
         \Tr(\mSigma)>n, 
    \end{equation}
    then 
    \begin{equation}
        \mathbb{E}[f(\vtheta_t)] > \mathbb{E}[f(\vpsi_t)].
    \end{equation}
\end{proposition}
We see that the preconditioner leading to a higher expected loss depends only on the power of the noise. 

Thus far, we analyzed the expected loss for a minimum point. In Section \ref{sec:Analysis for a Saddle Point} we consider the use of a preconditioner for a process that is initialized at a saddle point.

\subsection*{Saddle Points}\label{sec:Analysis for a Saddle Point}
For completeness, we analyze saddle points using our approach, while focusing on the ability of the multivariate OU process from (\ref{eq: OU process for dx_t}) to escape a saddle. To this end, here we assume that the initial point $\vtheta_0$ is a saddle point.
\begin{definition}[\textbf{Escape time}]\label{def: escape time}
    The escape time $t_{\text{esc}}$ from a saddle point, $\vtheta_0$, is defined  as the first time for which $\mathbb{E}[f(\vtheta_t)] < f(\vtheta_0)$, when the process is initialized at $\vtheta_0$.
\end{definition}

Under this definition, we have the following result.
\begin{proposition}[\textbf{Escaping a saddle point}]\label{prop: escape time from saddle point}
    The escape time from a saddle point, $t_{\text{esc}}$, is upper bounded by
    \begin{equation}\label{eq: t_esc expression}
        t_{\text{esc}} \leq
        \frac{\log\left(\frac{\Tr(\mSigma\mG)}{\lambda_{\min}\{\mSigma\mG\}}\right)}{|2\lambda_{\min}\{\mG\mH\}|}.
    \end{equation}
\end{proposition}
We note that the matrices $\mG\mH$ and $\mSigma\mG$ are similar to symmetric matrices so they have real spectra. In addition, for $\mG=\mI$ and $\mSigma = \mI$, the upper bound coincides with the result presented in \citet{chen2020stationary}. According to Proposition \ref{prop: escape time from saddle point}, for $t>t_{\text{esc}}$, it is guaranteed that $\mathbb{E}[f(\vtheta_t)] < f(\vtheta_0)$. Additionally, we see that the interaction between the preconditioning matrix and the Hessian has a greater effect on the escape time than the interaction between the preconditioning matrix and the noise covariance. We remark that escaping saddle points considering adaptive gradient methods was previously explored, for example in \cite{staib2019escaping}.

%% file: ExperimenralResults.tex
\section*{Application to Hessian Rank Estimation}\label{sec: Hessian Rank Estimation}
Consider a general NN with a vector parameter $\vtheta$. Suppose that the NN is already trained and that at the end of training the parameters are in the vicinity of some minimum $\vtheta_0$ of the loss $ \mathcal{L} $. Given $\vtheta_0$, our goal is to estimate the rank of the loss' Hessian at~$\vtheta_0$. 

In light of Corollary \ref{thm: E[f(x)] --> rank(H)}, we propose the following estimator of the Hessian rank
\begin{equation}\label{eq: r hat estimated  Hessian rank}
    \hat{r} = 
    \frac{4}{\sigma^2}\left(\langle\mathcal{L}(\vtheta_t)\rangle_{\mathcal{K}} - \mathcal{L}(\vtheta_0)\right),
\end{equation}
where $\mathcal{L}(\vtheta_t)$ is the loss of the NN at time $t$ and $\langle \mathcal{L}(\vtheta_t) \rangle_\mathcal{K}$ is the average of the loss over $t \in \mathcal{K}$ for the set of indices $\mathcal{K}$ (taken to be the last $K_{\text{avg}}$ iterations). We subtract $\mathcal{L}(\vtheta_0)$ since in general $\mathcal{L}(\vtheta_0)\neq 0$. Broadly, estimating the Hessian rank according to (\ref{eq: r hat estimated  Hessian rank}) requires the computation of $\langle \mathcal{L}(\vtheta_t) \rangle_\mathcal{K}$. To compute $\langle\mathcal{L}(\vtheta_t)\rangle_\mathcal{K}$, we update the parameters of the NN, $\vtheta_t$, according to a discretization of the SDE in~(\ref{eq: SDE G gradient G noise}), which is given by the Euler–Maruyama method \cite{kloeden1992stochastic}:
\begin{equation}
\label{eq: discrete Langevin with G}
    \vtheta_{t+1} = \vtheta_t - \eta_{t+1}\mG\nabla \mathcal{L}(\vtheta_t) + \sqrt{\eta_{t+1}}\mG\mSigma^{\frac{1}{2}} \vn_{t+1},
\end{equation}
where $\eta_t>0$ is the stepsize, $\nabla\mathcal{L}(\vtheta_t)$ is the gradient of the loss, and the preconditioner $\mG$ and noise covariance $\mSigma$ are chosen according to Corollary $\ref{thm: E[f(x)] --> rank(H)}$. Next, we describe the rank estimation algorithm in detail.

\subsection*{Algorithm}\label{sec:Algorithm}
\input{PracticalHessianRankEstimation.tex}

\subsection*{Linear Networks}\label{subsec: Linear Networks}
We start by considering linear NNs in a regression task since there exists an explicit expression for their Hessian at a global minimum. This allows us to evaluate our approach by comparing the Hessian rank estimation with the true rank.

The output of a linear NN is
\begin{equation}
    g_\vtheta(\vx) = \mW_M\mW_{M-1}\cdots\mW_1 \vx,
\end{equation}
where $\mW_i$ denotes the weights of the $i$th layer, and $M$ denotes the depth. Here the vector parameter $\vtheta$ is just a concatenation of the vectorizations of all the $\mW_i$ matrices. In this experiment, we use the quadratic loss, \emph{i.e.}
\begin{equation}
    \mathcal{L}(\vtheta) = \frac{1}{n}\sum_{j= 1}^{n}\left[ \| \vy_j -  g_\vtheta(\vx_j) \|^2 \right],
\end{equation}
where $ \{ ( \vx_j,\vy_j ) \}_{j= 1}^{n} $ is a paired training set. We denote by $d_x$ and $d_y$ the dimension of $\vx$ and $ \vy $, respectively. 

In this setting, there is no unique minimum, and the set of global minima is $\{\vtheta\in\sR^n : \mW_M\mW_{M-1}\cdots\mW_1=\mSigma_{xy}\mSigma_x^{-1}\}$, where $\mSigma_x$ is the covariance matrix of $\vx$ and $\mSigma_{xy}$ denotes the cross-covariance matrix between $\vx$ and $\vy$. The Hessian at a global minimum is given by \cite{mulayoff2020unique}
\begin{equation}\label{eq: h = phi phi}
    \mH = 2\mPhi\mPhi^\top,
\end{equation}
where $\mPhi = [\mPhi_1^\top, \mPhi_2^\top,\ldots,\mPhi_M^\top]^\top$ and 
\begin{equation}\label{eq: expression for phi}
    \mPhi_k = 
    \Bigg(\prod_{j=1}^{k-1}\mW_j\mSigma_x^{\frac{1}{2}} \Bigg) \otimes
    \Bigg(\prod_{j=k+1}^{M}\mW_j\mSigma_x^{\frac{1}{2}} \Bigg)^\top.
\end{equation}
Here  $\otimes$ denotes the Kronecker product. From (\ref{eq: h = phi phi}) and (\ref{eq: expression for phi}), \citet{mulayoff2020unique} showed that the Hessian rank equals the dimension of the input multiplied by the dimension of the output, \emph{i.e.}, $\rank(\mH) = d_xd_y$, regardless of the depth of the linear network. We leverage this known Hessian rank to evaluate the performance of our approach.
 
In all the experiments, we set the depth $M=5$, and w.l.o.g.\ consider $d_x = d_y = d$. The input $\vx$ is assumed to be random with a covariance matrix $\mSigma_x=\mI$. The cross-covariance matrix between $\vx$ and $\vy$, denoted by $\mSigma_{xy}$, is arbitrarily chosen at random.

Figure \ref{fig: Loss vs iterations linear Nets} presents the loss of a linear NN with input dimensions of $d=32$ and $Md^2=5120$ parameters. The weights are initialized at one of the global minima, randomly chosen. We set $\mG=\mI$,  $\eta = 10^{-4}$ and $\sigma_n^2 = 2\times 10^{-5}$. We remark that for nonlinear NNs, we found that using the Adam preconditioner is preferable to avoid stability issues. The actual loss is the blue curve and the theoretical values of the expected loss according to Theorem~\ref{thm: E[f(x)] expression} and Proposition~\ref{thm: E[f(x)] expression for large t} are the red and black curves, respectively. We see that the theoretical values are in accordance with the actual loss. We report that the expected loss over time according to Theorem \ref{thm: E[f(x)] expression} converges to the steady-state value in Proposition \ref{thm: E[f(x)] expression for large t} for a larger number of iterations.

Next, we evaluate the performance of the Hessian rank estimation for linear NNs with varying dimensions. We consider six different networks with input and output dimensions of $d=\{8, 16, 32, 64, 128,256\}$. The number of parameters is $Md^2$ and varies between $320$ parameters for $d=8$ to $327,680$ parameters for $d=256$. This allows us to evaluate the proposed approach for small-scale and large-scale linear NNs. For each network, we run $100$ different rank estimation trials, where at each trial the network is initialized at a randomly chosen global minimum. We follow Algorithm \ref{algo:Hessian rank estimaion} with $K_{\text{tot}} = 1.5\times 10^4$ and $K_{\text{avg}}=10^4$ iterations, noise power of $\sigma_2=2\times 10^{-5}$, and the stepsize is $\eta=10^{-4}$.  We compare the performance with the matrix rank estimation method proposed in \cite{ubaru2016fast} using their published implementation. For a fair comparison, we consider the same number of iterations for both methods, by setting the polynomial degree to $50$ and the number of vectors to $300$. For brevity, we term their method U\&S. We note that the presented results are not sensitive to different choices of hyperparameters. The U\&S method requires the Hessian-vector product. To avoid computing the entire Hessian, which for a network of $d=256$ has over $10^{11}$ entries, we exploit the known structure of the Hessian and the following property of the Kronecker product 
\begin{equation}
    \text{vec}(\mM_1\mM_2\mM_3) = (\mM_3^\top\otimes \mM_1)\text{vec}(\mM_2),
\end{equation}
for any matrices $\mM_1,\mM_2$, and $\mM_3$. This allows efficient computation of the Hessian-vector product. 

Figure \ref{fig: Est rank vs rank linear Nets} presents an error bar of the Hessian rank estimation obtained by the proposed approach in blue and the U\&S method in red. The whiskers represent the standard deviation (STD). The gray dashed line is the $y=x$ line representing the accurate rank. 
We see that the proposed approach leads to accurate rank estimation. In contrast, the U\&S method leads to underestimation. The reason is that U\&S is based on the estimation of the spectral density, and it requires a threshold that determines the rank. For the involved spectrum of a Hessian of NNs, setting this threshold is challenging. In contrast, the proposed approach does not require the estimation of the spectral density of the Hessian.
Additionally, since the rank estimate is proportional to the expected loss, the y-axis in Figure \ref{fig: Est rank vs rank linear Nets} is proportional to the expected loss, and there exists a linear relation between the Hessian rank and the expected loss as Corollary \ref{thm: E[f(x)] expression for large t} suggests. The small STD implies that the average loss of different trials is similar. Since each trial is initialized at a different minimum with a different Hessian, these results are in accordance with Corollary \ref{prop: E[f(X)] --> tr(Sigma G)} (even though the Hessian is PSD). 

Table \ref{Table: RMSE of Hessian rank estimation} presents the root mean square error (RMSE) normalized by the actual Hessian rank for the proposed approach and U\&S. We see that the proposed approach leads to a better accuracy by a large margin.   
\begin{table}[ht]
    \centering
    \setlength{\tabcolsep}{4.9pt}
    \begin{tabular}[t]{lcccccc}
        \toprule
        Rank& $8^2$ & $16^2$ & $32^2$ & $64^2$ & $128^2$ & $256^2$\\
        \midrule
        Ours  &6.9\%&2.8\%&1.2\%&0.3\%&2\%&4.6\%\\
        U\&S &12.2\%&18.5\%&19\%&19.2\%&17.9\%&18.7\%\\
        \bottomrule
    \end{tabular}
    \caption{Normalized RMSE of the Hessian rank estimation for different Hessian ranks.}
    \label{Table: RMSE of Hessian rank estimation}
\end{table}
We emphasize that the same hyperparameters are used for all the tested NNs, even though their number of parameters varies drastically. Potentially, improvement in performance could be achieved by adjusting the hyperparameter according to the NN dimension. 

\begin{figure}
    \centering
    \includegraphics[width=0.75\columnwidth]{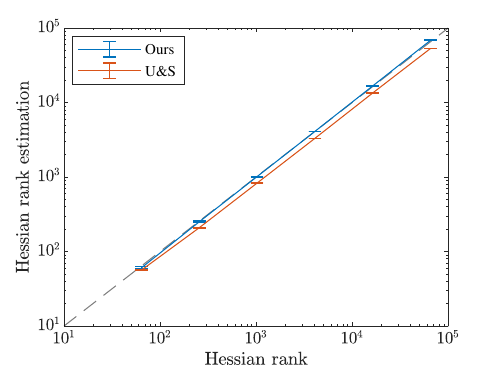} %
    \caption{Estimated rank of linear networks with different dimensions. Our method is in blue and the U\&S method proposed in \citet{ubaru2016fast} is in red. The error bars represent one standard deviation. We set the same number of iterations for both methods.}
    \label{fig: Est rank vs rank linear Nets}
\end{figure}

\subsection*{Denoising NN}
In this section, we demonstrate Algorithm~\ref{algo:Hessian rank estimaion} in a real-world application of nonlinear deep neural networks. We train a DnCNN \cite{zhang2017beyond} for denoising on the MNIST dataset. The network contains a few parameters ($753$) so that the exact computation of the Hessian and its spectrum for evaluation is feasible. In this experiment, we add a white Gaussian noise $ \mathcal{N}(\boldsymbol{0}, \boldsymbol{I}) $ to the training images and use the MSE loss. To avoid overfitting, in each epoch, we draw new noise realizations. We trained the network using SGD, and finalized the training using GD with a small stepsize to ensure reaching the minimum (for evaluation purposes). 

To estimate the Hessian rank at this minimum, we use Algorithm~\ref{algo:Hessian rank estimaion} with Adam preconditioner. Specifically, after convergence, we further train the model with Adam to get its second-moment estimation. Then, we fix the preconditioner $\mG $ of Adam, and run Algorithm~\ref{algo:Hessian rank estimaion} for $K_{\text{tot}} = 30\times 10^3$ iterations with stepsize $\eta=0.1$, and $\sigma^2 = 10^{-5}$. The last $K_{\text{avg}} = 10^4$ iterations are used to compute the averaged loss.

In this setup, there is no analytic expression for the Hessian rank. Thus, to evaluate the performance of Algorithm~\ref{algo:Hessian rank estimaion}, we numerically computed the Hessian using automatic differentiation. Due to finite numerical precision, it is expected that the rank of this numerically computed Hessian be falsely larger than the rank of the true Hessian (which we do not have) since eigenvalues that should theoretically be zero are replaced by small nonzero values. Therefore, instead of comparing the estimated rank to this falsely larger rank, we measure how much of the spectrum energy of the numerical Hessian is within the estimated rank. Specifically, we use the cumulative sum of eigenvalues and normalize it by the Hessian trace, namely, $\bar{\lambda}_j = \frac{1}{\Tr(\boldsymbol{H})}\sum_{i=1}^j\lambda_i\{\mH\}$.

Figure \ref{fig: Hessian eig and rank Non linear Net} presents the spectrum of the Hessian $\{\bar{\lambda}_j\}_{j=1}^{753}$ in blue, and the proposed Hessian rank estimation as a vertical black line. The dashed black line presents the Hessian rank estimation of the U\&S method. We see that our method captures more than $98.5\%$ of the energy of the eigenvalues with a rank estimate of $133$. In contrast, the U\&S method results in the rank estimation of $34$ capturing only $88\%$ of the energy. We note that in this experiment, using Adam preconditioner is required for stability. In the SM, we show that using the identity matrix as a preconditioner is unstable.

\begin{figure}
    \centering
    \includegraphics[width=0.75\columnwidth]{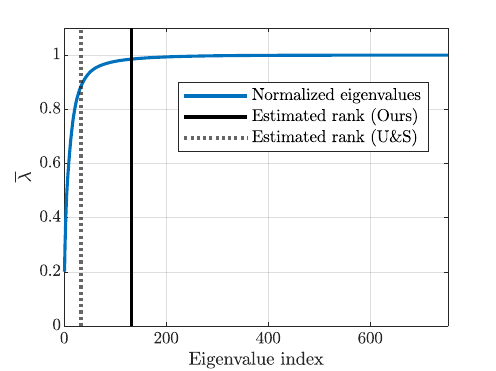}%
    \caption{The normalized cumulative sum of the eigenvalues of the Hessian for DnCNN (blue). The solid black line and the dashed gray line represent the estimated rank of the proposed approach and the U\&S method, respectively.}
    \label{fig: Hessian eig and rank Non linear Net}
\end{figure}

%% file: PracticalHessianRankEstimation.tex
The first step of the algorithm is choosing a preconditioner. We empirically found that for deep NNs, using $\mG=\mI$ which corresponds to GD could be unstable (see SM). To circumvent instability, we suggest using the weights of an adaptive gradient method at the minimum point (which in many cases is a byproduct of the training). In our experiments, we use Adam preconditioner to demonstrate the applicability of a commonly used adaptive method. We note that other adaptive gradient methods could also be used. We continue by setting the noise covariance according to $\mSigma=\sigma^2\mG^{-1}$. Next, the weights of the NN are updated according to (\ref{eq: discrete Langevin with G}) with the preconditioner $\mG$ and the noise covariance $\mSigma$ for $K_{\text{tot}}$ iterations. The gradient of the loss $\mathcal{L}(\vtheta_t)$ is computed using full batch with standard backpropagation. Finally, the average loss during the last $K_{\text{avg}}$ iterations out of the total $K_{\text{tot}}$ iterations is computed, and the Hessian rank is estimated using the average loss according to (\ref{eq: r hat estimated  Hessian rank}). 

Setting the hyperparameters $K_{\text{tot}}$ and $K_{\text{avg}}$ can be done by observing the curve of the actual loss during the update of~$\vtheta_t$. The computation of the average loss should begin when the loss reaches a steady-state. This, empirically, does not require many iterations. For example, in the experimental results, we set $K_{\text{tot}} = 1.5\times 10^4$ and $K_{\text{avg}}=10^4$. The algorithm is summarized in Algorithm \ref{algo:Hessian rank estimaion}.

\begin{algorithm}[tb]
    \caption{Hessian rank estimation}
    \label{algo:Hessian rank estimaion}    
    \textbf{Input}: $\vtheta_0$ (NN weights at the end of the training) \\
    \textbf{Parameter}: $\sigma^2$, $K_{\text{tot}}$, $K_{\text{avg}}$ \\
    \textbf{Output}: $\hat{r}$ (rank estimate)
    \begin{algorithmic}[1]   
        \STATE Choose a fixed preconditioner $\mG$
        \STATE Set $\mSigma=\sigma^2\mG^{-1}$
        \STATE Update $\vtheta_t$ according to (\ref{eq: discrete Langevin with G}) with $\mG$ and $\mSigma$ for $K_{\text{tot}}$ iterations
        \STATE Compute $\langle\mathcal{L}(\vtheta_t)\rangle_{\mathcal{K}}$ over the last $K_{\text{avg}}$ iterations
        \STATE Compute $\hat{r}$ using (\ref{eq: r hat estimated  Hessian rank})
        \RETURN $\hat{r}$
    \end{algorithmic}
\end{algorithm}

In terms of complexity, the proposed rank estimation algorithm requires the computation of the full gradient at each iteration. 
In comparison, the matrix rank estimation of \citet{ubaru2016fast} requires a full Hessian-vector product at each iteration. This involves computing second-order derivatives or approximating the full Hessian-vector product by a numeric computation using backpropagation at two adjacent points. This results in twice the amount of computations at each iteration compared to the proposed approach.

Beyond the scope of NNs, we remark that the proposed approach could be used to estimate the rank of any symmetric matrix $\mS$ by computing the expected loss of the discretization of LD given in (\ref{eq: discrete Langevin with G}) and setting $\nabla \mathcal{L}(\vtheta_t) = \mS \vtheta_t$. However, this direction requires further investigation, exceeding the scope of this paper.

%% file: RelatedWork.tex
\section*{Related Work}
In \cite{zhu2019anisotropic}, a similar setting to ours, of a quadratic approximation of the loss function near minimum points, is considered. There, the OU process without a preconditioner is analyzed using Taylor expansion assuming a short time scale. In contrast to \cite{zhu2019anisotropic}, we consider a preconditioner in the OU process, and our analysis is for a large time scale. This allows us to examine the effect of the preconditioner as well as the Hessian on the expected loss.

The spectral properties of the Hessian of NNs were investigated in previous works. In \cite{ghorbani2019investigation}, an estimation of the spectral density of the Hessian is proposed, and insights regarding the geometry of the loss surface are derived. In \cite{yao2020pyhessian}, a framework for estimating the trace of the Hessian, its top eigenvalues, and its spectral density is also proposed, and the effect of the NN architecture on the Hessian spectrum is investigated. 
Nonetheless, to the best of our knowledge, a method for estimating the Hessian rank has not been presented.

Methods for estimating the rank of a matrix beyond the context of NNs have also been proposed \cite{perry2010minimax, kritchman2009non, ubaru2016fast}. However, they are application-specific or not well-suited for the Hessian of a NN. In \cite{ubaru2016fast}, it was proposed to estimate the Hessian rank by estimating the trace of the eigen projector, which, in turn, is estimated using a polynomial filter of degree $m$, $\phi_m(x)$, that serves as an approximation of a step function. 

Following Theorem \ref{thm: E[f(x)] expression} and setting $\mSigma = \mG = \mI$, the proposed approach uses the expected loss to compute the value of $\Tr{(\mI - e^{-\mH t})}$ for a large $t$. This could be viewed as the filter of $\phi(x) = 1-e^{-xt}$. Consequently, the proposed approach could be viewed as a means of using LD to implement an exponential filter. In contrast to the filters proposed by \citet{ubaru2016fast}, which are only polynomial approximations, the proposed approach implements the exact expression of the filter under the ergodicity assumption. Additionally, the U\&S method requires the estimation of the density of the spectrum as well as setting multiple hyperparameters, which affect its performance. For NNs, it requires the computation of the Hessian matrix or its approximation to produce the matrix-vector products. In contrast, the proposed approach relies only on the computation of the expected loss and is more suitable for estimating the rank of the Hessian of NNs. Furthermore, the proposed approach has a theoretical guarantee that in the limit of infinite iterations, it converges to the Hessian rank.

%% file: Conclusions.tex
\section*{Conclusions}
In this work, we consider preconditioned Langevin dynamics near stationary points and analyze the expected loss for different preconditioners. We show that when a preconditioner admits a particular relation with respect to the noise covariance, the expected loss is proportional to the Hessian rank. Following this theoretical result, we devise an iterative algorithm for neural network Hessian rank estimation. In addition, we show that under a certain condition, the expected loss at a large time scale depends only on the interplay between the preconditioner and the noise covariance. We use our analysis to compare SGD-like and Adam-like preconditioners, deriving a condition on which preconditioner leads to a higher loss. We empirically demonstrate the theoretical results and the accurate Hessian rank estimation for linear neural networks as well as DnCNN networks.

%% file: Acknowledgment.tex
\section*{Acknowledgments}
The work of AB and RT was supported by the European Union’s Horizon 2020 research and innovation programme under grant agreement No. 802735-ERC-DIFFO. The work of TM was partially supported by the Israel Science Foundation (grant no. 2318/22), and by the Ollendorff Minerva Center, ECE faculty, Technion. The work of RM was supported by the Planning and Budgeting Committee of the Israeli Council for Higher Education.

%% file: Proofs.tex
\section{Proofs}

\subsection{Proof of Theorem \ref{thm: E[f(x)] expression}}
For convenience, we repeat the SDE in (\ref{eq: OU process for dx_t})
\begin{equation}
         d\vtheta_t = -\mG\mH \vtheta_tdt +  \mG \mSigma^{\frac{1}{2}} d\vn_t.
\end{equation}
So, $\vtheta_t$ follows the multivariate OU process explicit formula
\begin{equation}
    \vtheta_t = \int_0^t e^{-\mG\mH(t-s)} \mG \mSigma^{\frac{1}{2}} d\vn_s,
\end{equation}
with the following correlation matrix
\begin{align}\label{eq: expression for E[x_tx_t^T]}
    \mathbb{E}[\vtheta_t\vtheta_t^\top ] 
    & = \mathbb{E}\left[ \int_0^t e^{-\mG\mH(t-s)}\mG \mSigma^{\frac{1}{2}} d\vn_s
    \left( \int_0^t e^{-\mG\mH(t-s)}\mG \mSigma^{\frac{1}{2}} d\vn_s \right) ^\top \right] \nonumber  \\
    & = \mathbb{E}\left[ \int_0^t e^{-\mG\mH(t-s)}\mG \mSigma^{\frac{1}{2}} (\mSigma^{\frac{1}{2}})^\top \mG^\top   e^{-\mH\mG(t-s)} ds \right] \nonumber \\
    & = \mathbb{E}\left[ \int_0^t e^{-\mG\mH(t-s)} \mG \mSigma \mG   e^{-\mH\mG(t-s)}
    ds \right],
\end{align}
where Ito's isometry is used. This is a well-known result of the multivariate OU process \cite{gardiner1985handbook}.

We examine the expected loss.
\begin{equation}
        \mathbb{E}[f(\vtheta_t)] =      
        \mathbb{E}\left[\frac{1}{2}\vtheta_t^\top \mH\vtheta_t\right] = 
        \frac{1}{2}\mathbb{E}\left[\Tr\left(\mH\vtheta_t\vtheta_t^\top \right)\right] = 
        \frac{1}{2}\Tr\left(\mH\mathbb{E}\left[\vtheta_t\vtheta_t^\top \right]\right) ,
\end{equation}
which follows from the linearity of the trace and the expectation. 
We use (\ref{eq: expression for E[x_tx_t^T]}) and notice that the integral is $ds$, and $\mH$ depends on $s$, to obtain
\begin{equation}
    \begin{split}
        \mathbb{E}[f(\vtheta_t)] =  
        \frac{1}{2}
        \int_0^t \Tr\left( \mH e^{-\mG\mH(t-s)}\mG \mSigma \mG  e^{-\mH\mG(t-s)}
        ds \right).
    \end{split}
\end{equation}
According to Lemma~\ref{lemma: matrix similarity}, the matrix inside the trace, $ \mH e^{-\mG\mH(t-s)}\mG \mSigma \mG  e^{-\mH\mG(t-s)} $, is similar to $ \mH e^{-2\mG\mH(t-s)} \mG \mSigma\mG $, so they share their spectra and have the same trace. We therefore get
\begin{equation}
    \mathbb{E}[f(\vtheta_t)]
    = \mathbb{E}\left[\mathrm{\frac{1}{2}}\vtheta_t^\top \mH \vtheta_t\right] 
    = \mathrm{\frac{1}{2}}\int_0^t\Tr\left(e^{-2\mG\mH(t-s)}\mG  \mSigma\mG  \mH \right)ds
    = \frac{1}{4}\Tr\left(\mSigma \mG\left(\mI-e^{-2\mG\mH t}\right)\right).
\end{equation}

\subsection{Proof of Proposition \ref{thm: E[f(x)] expression for large t}}
First, we note that the matrix product $\mG\mH$ is similar to $\mG^{\frac{1}{2}}\mH\mG^{\frac{1}{2}}$, so it has a real spectrum and it can be expressed by 
\begin{equation}
    \mG\mH = \mP\mLambda\mP^{-1},
\end{equation}
where $\mLambda$ is a diagonal matrix containing the eigenvalues of $\mG\mH$ and $\mP$ is a matrix whose columns are the eigenvectors of $\mG\mH$. According to Lemma~\ref{lemma: GHG is SPD}, the matrix $\mG^{\frac{1}{2}}\mH\mG^{\frac{1}{2}}$ is PSD so the eigenvalues of $\mG\mH$ are greater or equal to zero as well, namely, $\mLambda_{ii}\ge 0$. 
We have that $e^{-2\mG\mH t} = \mP e^{-2\mLambda t} \mP^{-1}$ and 
\begin{equation}
    \lim_{t\rightarrow\infty}e^{-2\mLambda_{ii} t} = 
    \begin{cases}
         1 & \mLambda_{ii} = 0\\
         0 & \mLambda_{ii} \neq 0.
    \end{cases}
\end{equation}    
By Theorem \ref{thm: E[f(x)] expression} and since $\mI = \mP\mP^{-1}$ we have
\begin{equation}
    \lim_{t\rightarrow\infty}\mathbb{E}[f(\vtheta_t)]
    = \lim_{t\rightarrow\infty}\frac{1}{4}\Tr\left(\mSigma \mG\mP\left(\mI-e^{-2\mLambda t}\right)\mP^{-1}\right)
    = \lim_{t\rightarrow\infty}\frac{1}{4}\Tr\left(\mSigma \mG\mP\mJ\mP^{-1}\right) .
\end{equation}

\subsection{Proof of Proposition \ref{prop: rank H and loss}}
By Proposition~\ref{thm: E[f(x)] expression for large t} we have
\begin{align}
    \lim_{t\rightarrow\infty}\mathbb{E}[f_2(\vtheta_t)] - \lim_{t\rightarrow\infty}\mathbb{E}[f_1(\vtheta_t)] 
    & = \lim_{t\rightarrow\infty}\frac{1}{4}\Tr\left(\mSigma \mG \left( e^{-2\mG\mH_1 t} - e^{-2\mG\mH_2 t}\right)\right) \nonumber \\
    & = \frac{1}{4}\Tr\left(\mSigma \mG \left( \mP_1\mJ_1\mP_1^{-1} - \mP_2\mJ_2\mP_2^{-1}\right)\right),
\end{align}
where $\mG\mH_i =  \mP_i\mJ_i\mP_i^{-1}$ for $i=1,2$.

Using Sylvester's inequality \cite{petersen2008matrix} with the PD matrix $\mG$, we have that $\text{rank}(\mG\mH_i) = \text{rank}(\mH_i)$ for $i=1,2$. Since $\rank(\mH_1) \le \rank(\mH_2)$ it holds that
\begin{equation}
    \mP_1\mJ_1\mP_1^{-1} \succeq
   \mP_2\mJ_2\mP_2^{-1}.
\end{equation}
So, $\mM = \mP_1\mJ_1\mP_1^{-1} - \mP_2\mJ_2\mP_2^{-1}$ is a PSD matrix. Using Lemma \ref{Lemma: trace inequalities lemma} twice with the matrix product $\mSigma\mG\mM$ completes the proof.

\subsection{Proof of Proposition~\ref{prop: Ef(x) > Ef(y) two preconditioning}}
We denote the preconditioners of $\vtheta_t$ and $\vpsi_t$ by $\mG_1$ and $\mG_2$, respectively. We consider $\mG_1=\sqrt{\frac{\Tr(\mSigma^{-1})}{n}}\mI$ and $\mG_2=\mSigma^{-\frac{1}{2}}$ so that $||\mG_1||_F^2 = ||\mG_2||_F^2 = \Tr(\mSigma^{-1})$. According to Corollary \ref{thm: E[f(x)] expression for large t} the following holds 
\begin{equation}\label{eq: E[f(x)] for large t for sgd with normalization}
    \lim_{t\rightarrow\infty}\mathbb{E}[f(\vtheta_t)]
    = \Tr(\mG_1\mSigma)
    = \Tr\left(\sqrt{\frac{\Tr(\mSigma^{-1})}{n}}\mI \mSigma \right)
    = \sqrt{\frac{\Tr(\mSigma^{-1})}{n}}\Tr(\mSigma).     
\end{equation}

We continue by using the inequality of quadratic mean and arithmetic mean (part of the HM-GM-AM-QM inequalities).
\begin{align}
    \lim_{t\rightarrow\infty}\mathbb{E}[f(\vtheta_t)]
    & = \sqrt{\frac{\Tr(\mSigma^{-1})}{n}}\Tr(\mSigma) \nonumber \\ 
    & = \Tr(\mSigma) \sqrt{\frac{\frac{1}{\lambda_1}+\frac{1}{\lambda_2}+...+\frac{1}{\lambda_d}}{n}} \nonumber \\
    & \geq \Tr(\mSigma) \frac{\sqrt{\frac{1}{\lambda_1}}+\sqrt{\frac{1}{\lambda_2}}+...+\sqrt{\frac{1}{\lambda_d}}}{n} \nonumber \\
    & = \frac{\Tr(\mSigma)}{n}\cdot\Tr(\mSigma^{-\frac{1}{2}}) \nonumber \\
    & \ge \Tr(\mSigma^{-\frac{1}{2}}) \nonumber \\
    & = \lim_{t\rightarrow\infty}\mathbb{E}[f(\vpsi_t)],
\end{align}
where we also use $\frac{\Tr(\mSigma)}{n}>1$.

\subsection{Proof of Proposition~\ref{prop: escape time from saddle point}}
We use Theorem~\ref{thm: E[f(x)] expression} and  Lemma~\ref{lemma: tr > lambda_min lmbda_max} to obtain
\begin{align}\label{eq: E[F(x)] inequalities bounds}
    \mathbb{E}[f(\vtheta_t)]
    & = \frac{1}{4}\Tr\left(\mSigma \mG\right) - \frac{1}{4}\Tr\left(\mSigma \mG e^{-2\mG\mH t}\right) \nonumber \\ 
    & \leq \frac{1}{4}\Tr(\mSigma \mG) - \frac{1}{4}\lambda_{\min}\{\mSigma \mG\}\lambda_{\max}\{e^{-2\mG\mH t}\} \nonumber \\
    & = \frac{1}{4}\Tr(\mSigma \mG) - \frac{1}{4}\lambda_{\min}\{\mSigma \mG\}e^{-2\lambda_{\min}\{\mG\mH \}t}.
\end{align}

The escape time occurs when the expectation is negative, and this happens for
\begin{equation}
    t_{\text{esc}} 
    > \frac{\log\left(\frac{\Tr(\mSigma \mG)}{\lambda_{\min}\{\mSigma \mG\}}\right)}{-2\lambda_{\min}\{\mG\mH\}}.
\end{equation}
We note that $\lambda_{\min}\{\mG\mH\}<0$ due to Lemma \ref{lemma: GH negative eigenvalue} so the r.h.s of the inequality is positive. Furthermore, the matrices $\mSigma\mG$ and $\mG\mH$ are similar to symmetric matrices, so they have a real spectrum.

%% file: AuxiliaryLemmas.tex
\section{Auxiliary Lemmas}
\begin{lemma}\label{lemma: matrix similarity}
    The matrix $ \mM_1 = \mH e^{-\mG\mH(t-s)}\mG\mSigma \mG e^{-\mH\mG(t-s)}$ is similar to $\mM_2 = \mH e^{-2\mG\mH(t-s)} \mG \mSigma\mG $.
\end{lemma}
\begin{proof}
We start with the matrix $\mM_1$ and multiply it by $\mP = e^{\mH\mG(t-s)}$ from the right and by $\mP^{-1} = e^{-\mH\mG(t-s)}$ from the left ($\mP$ is square and invertible) to obtain
\begin{equation}\label{eq: matrix tmp}
    \mP \mM_1 \mP^{-1}
    = e^{-\mH\mG(t-s)} \mH e^{-\mG\mH(t-s)}\mG\mSigma \mG .
\end{equation}
We notice that
\begin{equation}\label{eq: exp*H H*exp}
    e^{-\mH\mG(t-s)}\mH
    = \left(\sum_k\frac{1}{k!}\left(-\mH\mG(t-s)\right)^k\right)\mH
    = \mH\left(\sum_k\frac{1}{k!}\left(-\mG\mH(t-s)\right)^k\right)
    = \mH e^{-\mG\mH(t-s)},
\end{equation}
since
\begin{equation}
    [(\mH\mG)^k]\mH = [\mH\mG\mH\mG\cdots\mH\mG]\mH = \mH[\mG\mH\mG\mH\cdots\mG\mH] = \mH[(\mG\mH)^k].
\end{equation}
We continue from (\ref{eq: matrix tmp}) using (\ref{eq: exp*H H*exp}) to get
\begin{align}
    \mP \mM_1 \mP^{-1}
    & = \mH e^{-\mG\mH(t-s)}  e^{-\mG\mH(t-s)}\mG\mSigma \mG \nonumber  \\
    & = \mH e^{-2\mG\mH(t-s)} \mG\mSigma \mG \nonumber \\
    & = \mM_2.
\end{align}
Since $\mM_1 = \mP^{-1}\mM_2\mP$ the two matrices are similar.    
\end{proof}

\begin{lemma}\label{lemma: GHG is SPD}
    If $\mH$ is a PSD matrix and $\mG$ is a full-rank symmetric matrix then
    $\mG\mH\mG$ is a PSD matrix.
\end{lemma}
\begin{proof}
    For every vector $\vu\neq \mathbf{0}$ we have
    \begin{equation}
          \vu^\top(\mG\mH\mG)\vu = (\mG\vu)^\top\mH(\mG\vu) \geq 0,
    \end{equation}
    since $\mH$ is PSD and $\mG\vu\neq \mathbf{0}$.
\end{proof}

\begin{lemma}\label{lemma: tr > lambda_min lmbda_max}
    If $\mH$ is a symmetric matrix and $\mSigma$ and $\mG$ are PD matrices, 
    then the following holds
    \begin{equation}
      \Tr(\mSigma \mG e^{-2\mG\mH t} ) > \lambda_{\min}\{\mSigma \mG\} \lambda_{\max}\{e^{-2\mG\mH t}\}
    \end{equation}
\end{lemma}
\begin{proof}
    Since the matrix $\mG\mH$ is similar to the matrix $\mG^{\frac{1}{2}}\mH\mG^{\frac{1}{2}}$, the matrix $e^{-2\mG\mH t}$ has only positive eigenvalues, i.e., $e^{-2\mG\mH t}\succ \mathbf{0}$. Similarly, the matrix $\mSigma\mG$ is similar to the PD matrix $\mG^{\frac{1}{2}}\mSigma\mG^{\frac{1}{2}}$, sharing its positive spectrum. Thus we have 
    \begin{equation}
        \mSigma\mG - \lambda_{\min}\{\mSigma\mG\}\mI 
        \succeq
        \mathbf{0},
    \end{equation}
    and
    \begin{equation}
        \Tr \left( (\mSigma\mG - \lambda_{\min}\{\mSigma\mG\}\mI ) e^{-2\mG\mH t} \right)
        = \Tr(\mSigma \mG e^{-2\mG\mH t} ) - \lambda_{\min}\{\mSigma\mG\}
        \Tr(e^{-2\mG\mH t}) \geq 0.
    \end{equation}
    So,
    \begin{equation}
        \Tr(\mSigma \mG e^{-2\mG\mH t} )
        \geq \lambda_{\min}\{\mSigma\mG\} \Tr(e^{-2\mG\mH t}) 
        > \lambda_{\min}\{\mSigma\mG\} \cdot \lambda_{\max}\{e^{-2\mG\mH t}\}.
    \end{equation}
\end{proof}

\begin{lemma}\label{lemma: GH negative eigenvalue}
    Let $\lambda_1<0$ be a negative eigenvalue of the matrix $\mH$ associated with the eigenvector $\vu_1$. If $\mG$ is PD, then the matrix product $\mG\mH$ has a negative eigenvalue.
\end{lemma}
\begin{proof}
    First, we note that $\mG \mH$ is similar to a symmetric matrix $\mG^{\frac{1}{2}}\mH\mG^{\frac{1}{2}}$. We define $\vv=\mG^{-\frac{1}{2}}\vu$ and examine
    \begin{equation}
        \vv^T \mG^{\frac{1}{2}}\mH\mG^{\frac{1}{2}} \vv
        = \vu_1^T\mH \vu_1<0,    
    \end{equation}
    so the matrix the matrix $\mG^{\frac{1}{2}}\mH\mG^{\frac{1}{2}}$ has at least one negative eigenvalue. Since $\mG\mH$ is similar to $\mG^{\frac{1}{2}}\mH\mG^{\frac{1}{2}}$, they have the same spectrum and $\mG\mH$ also has a negative eigenvalue.
\end{proof}

\begin{lemma}\label{Lemma: trace inequalities lemma}
Let $\mP$ be a PD matrix, and $\mS$ be a symmetric matrix. Then, the following holds
\begin{align}
        \Tr(\mP\mS)
        & \leq \lambda_{\max}\{\mS\}\Tr(\mP)\\
        \Tr(\mP\mS)
        & \geq \lambda_{\min}\{\mS\}\Tr(\mP) \nonumber \\
        & \geq \lambda_{\min}\{\mS\}\lambda_{\max}\{\mP\} .
\end{align}
\end{lemma}
\begin{proof}
    Since $\mS$ is a symmetric matrix it has a real spectrum and it holds that $\lambda_{\max}\{\mS\}\mI-\mS\succeq \mathbf{0}$. We continue by
    \begin{equation}
        \Tr{(\mP{(\lambda_{\max}\{\mS\}\mI-\mS)})} = 
        \lambda_{\max}\{\mS\}\Tr{(\mP)} - \Tr{(\mP\mS)} \geq 0.
    \end{equation}
    Similarly, since $\mS - \lambda_{\min}\{\mS\}\mI \succeq \mathbf{0}$, we have
    \begin{equation}
        \Tr{(\mP{(\mS - \lambda_{\min}\{\mS\}\mI)})} = 
        \Tr{(\mP\mS)} - \lambda_{\min}\{\mS\}\Tr{(\mP)}\geq 0.
    \end{equation}
\end{proof}

%% file: HessianTraceEstimation.tex
\section{Connection to Hessian Trace Estimation}
Recall that according to (\ref{eq: d/dt of Ef(x) = trace H}), it holds that $\frac{\partial}{\partial t} \mathbb{E}[f(\vtheta_t)]\Big|_{t=0} =  \frac{1}{2}\Tr(\mH )$. According to the weights update in (\ref{eq: discrete Langevin with G}) for $\vtheta_0 = \mathbf{0}$ we have
\begin{equation}
    \vtheta_{1} = \sqrt{\eta}\vn_{1},
\end{equation}
which means that $\vtheta_1\sim\mathcal{N}(\mathbf{0},\eta \mI)$. 

The discrete implementation of (\ref{eq: d/dt of Ef(x) = trace H}) becomes 
\begin{align}
    \frac{1}{2}\Tr(\mH)
    = \frac{\partial}{\partial t} \mathbb{E} [f(\vtheta_t)]\Big|_{t=0}
    & \approx \left<\frac{f(\vtheta_1)-f(\vtheta_0)}{\eta}\right> \nonumber \\ 
    & = \frac{1}{\eta}\left<\frac{1}{2}\vtheta_1^T\mH\vtheta_1 - \frac{1}{2}\vtheta_0^T\mH\vtheta_0\right> \nonumber \\
    & = \frac{1}{2}\left<\vn_1^T\mH\vn_1\right> ,    
\end{align}
where $\left<\cdot\right>$ is the empirical mean. We obtain the trace estimation method proposed in \citet{hutchinson1990stochastic}. We remark that this analysis offers a slightly different view, since instead of computing the quadratic term $\mathbb{E}[\vn^T\mH\vn]$, we can use the output of the network for the Hessian trace estimation.

%% file: AdditionalExperimentalResults.tex
\section{Additional Experimental Results}
In the denoising NN experiment, we use the DnCNN \cite{zhang2017beyond} architecture with $8$ channels and $3$ convolution layers. We use the ReLU activation and batch normalization and optimize the model using SGD. For setting the preconditioner $\mG$, we further train the model using Adam with zero momentum, \emph{i.e.} $\beta_1=0$, and $\beta_2 = 0.999$ for $10^3$ iterations. We note that this phase of training using Adam is only required for obtaining its second-order moment estimation to use as a preconditioner, and this setting is designed to showcase that our method can be applied to models trained by SGD with Adam preconditioner. However, in cases where Adam is used for the optimization, its second-order moment estimation at the final iteration can be used without further training. The stepsize for updating the weights according to LD in (\ref{eq: discrete Langevin with G}) is set to $\eta=0.1$.

Next, we demonstrate that using $\mG=\mI$ as a preconditioner could lead to instability. We repeat the experiment and update the weights according to (\ref{eq: discrete Langevin with G}) with $\mG=\mI$ and a smaller stepsize of $\eta=0.02$. Figure~\ref{fig: Loss using SGD} presents the loss. We see that after only a few iterations the loss reaches a very large value implying that the process has escaped the minimum.

\begin{figure}[h]
    \centering
    \includegraphics[width=0.4\columnwidth]{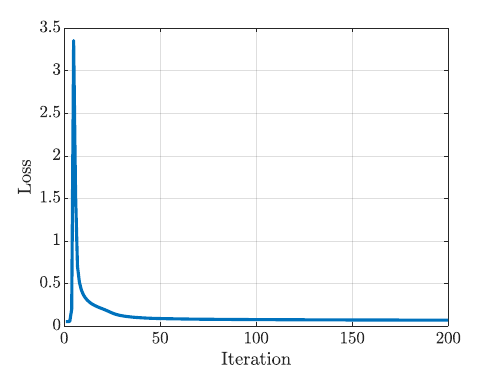}%
    \caption{The loss for DnCNN with $\mG=\mI$.}
    \label{fig: Loss using SGD}
\end{figure}